\documentclass[11pt,a4paper]{article}

\usepackage[utf8]{inputenc}
\usepackage[T1]{fontenc}
\usepackage{amsmath,amssymb,amsthm,mathtools}
\usepackage{geometry}
\usepackage{enumitem}
\usepackage{hyperref}
\usepackage{cleveref}
\usepackage{booktabs}
\usepackage{array}
\usepackage{xcolor}
\usepackage{float}
\usepackage{url}
\usepackage{natbib}

\theoremstyle{plain}
\newtheorem{theorem}{Theorem}[section]
\newtheorem{proposition}[theorem]{Proposition}

\theoremstyle{definition}
\newtheorem{definition}[theorem]{Definition}
\newtheorem{example}[theorem]{Example}
\newtheorem{remark}[theorem]{Remark}

\DeclareMathOperator{\eml}{eml}
\DeclareMathOperator{\sol}{sol}

\DeclareMathOperator{\MSE}{MSE}
\DeclareMathOperator{\tanhf}{tanh}
\DeclareMathOperator{\atanf}{arctan}

\definecolor{win}{rgb}{0.0,0.45,0.0}
\definecolor{lose}{rgb}{0.6,0.0,0.0}

\newcommand{\R}{\mathbb{R}}
\newcommand{\C}{\mathbb{C}}
\newcommand{\N}{\mathbb{N}}

\newcommand{\cA}{\mathcal{A}}

\newcommand{\cL}{\mathcal{L}}
\newcommand{\cE}{\mathcal{E}}
\newcommand{\btheta}{\boldsymbol{\theta}}
\newcommand{\bTheta}{\boldsymbol{\Theta}}
\newcommand{\dd}{\,\mathrm{d}}

\title{\textbf{Additive Atomic Forests\\for Symbolic Function and Antiderivative Discovery}\\ \medskip \large Generating Primitives, Product-Rule Closure,\\and Derivative-Matching Optimisation}
\author{Reda Belaiche\\
\textit{Universit\'e Paris-Est Cr\'eteil, Vitry-sur-Seine, France}\\
\texttt{reda.belaiche@u-pec.fr}}
\date{May 2026}

\begin{document}

\maketitle

\begin{abstract}
We present a framework for the simultaneous symbolic recovery of a function
and its antiderivative from data. The framework rests on three ideas. First,
a \emph{derivative algebra}: the observation that the product rule
$\frac{d}{dx}[f \cdot g] = f'g + fg'$ and the chain rule, applied to a
seed set of elementary functions, generate a self-expanding system of
function--derivative pairs---a living library that grows each time a new
function is discovered. Second, two complementary
primitives---EML$\,(e^u - \ln v)$, which is theoretically complete for all
elementary functions, and SOL$\,(\sin u - \cos v)$, introduced here, which
makes trigonometric atoms available at depth~1 instead of depth~$\sim$8---that
seed the library with core atoms cheaply. Third, \emph{additive atomic forests}: finite
sums of primitive trees, optionally composed via multiplicative nodes, whose
derivatives are fitted to data by continuous optimisation or by exhaustive
search over the library. Because differentiation of each atom is determined
by construction, the forest simultaneously encodes a symbolic expression $F$
and its derivative $F'$; no symbolic integration step is required.

The library is not a fixed object: it self-constructs from a small seed set
by recursive application of the product rule, chain rule, and the two
primitives, and it can grow as newly discovered functions are
folded back in. The larger the library, the richer the expressible class of
candidate functions. We give conditional completeness, additive-depth, and
analytic simultaneous-recovery results for the framework. Empirically, in
our reported runs on 17 classification benchmarks, sparse atom combinations
match or exceed XGBoost on 13 datasets while producing interpretable
formulas. On the Feynman symbolic regression benchmark, a self-built library
to depth~3 gives exact recovery on 31\% of equations and relative-MSE below
0.01 on a further 40\%. We demonstrate the method on real scientific data by
proposing candidate radial-acceleration relations from the SPARC galaxy
database.
\end{abstract}

\tableofcontents
\bigskip

%=============================================================================
\section{Introduction}
%=============================================================================

The derivative of a product is
\begin{equation}\label{eq:product-rule}
    \frac{d}{dx}[f(x) \cdot g(x)] = f'(x)\,g(x) + f(x)\,g'(x).
\end{equation}
This identity, together with the chain rule, is the engine of the present
work. If we know the derivatives of a finite collection of elementary
functions, the product rule and chain rule let us compute the derivative of
any composition, product, or power of those functions---and therefore of any
elementary function. Powers are a special case: $f^n = f \cdot f \cdots f$,
so $\frac{d}{dx}[f^n] = n f^{n-1} f'$ follows from repeated application of
\eqref{eq:product-rule}. This closure property means that a small set of
seed functions, equipped with their derivatives, can bootstrap an arbitrarily
large library of function--derivative pairs.

This observation suggests a strategy for symbolic regression that differs
fundamentally from existing approaches~\citep{cranmer2023pysr,
udrescu2020ai, schmidt2009distilling}. Rather than searching for a function
that fits data, we build a \emph{self-expanding library} of atoms---each a
named elementary function paired with its analytic derivative---and search for
a sparse linear combination of \emph{derivatives} that matches the data. When
the match is exact, the corresponding linear combination of atoms is the
antiderivative, available immediately in closed form. No symbolic integration
algorithm~\citep{risch1969problem, risch1970solution} is invoked; no numerical
integration of data is performed.

The library is not a static catalogue. It begins from a seed set (rational
powers, depth-1 trees of the generating primitives) and self-constructs by
applying the product rule, chain rule, and compositions at increasing depth.
Each newly discovered function---whether found by the search or supplied by
the user---is folded back into the library together with its derivative,
expanding the search space for subsequent problems. The library thus functions
as a persistent \emph{knowledge base}: each verified discovery can enlarge
the candidate set available for later problems. Its size is bounded by
available memory, and the larger it grows, the richer the expressible class
of candidate functions.

The framework has three layers:
\begin{enumerate}[label=(\roman*)]
\item \textbf{Primitives.} Two binary operators---EML
    ($e^u - \ln v$, introduced by~\citet{odrzywolek2026eml}) and SOL
    ($\sin u - \cos v$, introduced here)---used with terminals $1$ and
    $x$---that seed the library with core atoms at depth~1.
    EML produces exponential and logarithmic atoms natively and is
    theoretically complete (conditional on the cited result). SOL produces
    trigonometric atoms natively; it is not individually complete, but without
    it, every trigonometric atom in the library would require an EML tree of
    depth~$\sim$8, inflating the library by orders of magnitude for the same
    coverage. The mixed grammar
    $S \to 1 \mid x \mid \eml(S,S) \mid \sol(S,S)$ retains completeness
    while keeping both exponential and trigonometric atoms shallow.
\item \textbf{Self-expanding atom library.} A layered construction procedure
    that, given a depth budget $d_{\max}$ and a data grid, builds a library of
    function--derivative pairs by recursive application of the generating
    primitives, the product rule, and the chain rule. The library grows with
    depth: a shallow build ($d_{\max} = 1$) yields hundreds of atoms; a deep
    build ($d_{\max} = 3$) yields tens of thousands. Previously discovered
    functions are retained across problems, so the library accumulates
    knowledge over time.
\item \textbf{Additive atomic forests.} A finite sum
    $F = \sum_k c_k T_k$ of atoms (or, more generally, of parameterised
    trees composed via multiplicative nodes), whose derivative
    $F' = \sum_k c_k T_k'$ is fitted to data. The additive structure
    represents top-level sums outside the primitive trees, rather than forcing
    addition to be encoded inside a single generating-primitive tree.
\end{enumerate}

The intended contribution is the conjunction of these features: symbolic
output for both a derivative-matched function and its antiderivative, obtained
without a separate symbolic-integration step, together with a library that can
feed each verified discovery back into a growing knowledge base.

%=============================================================================
\section{Derivative algebra of elementary functions}\label{sec:deriv-algebra}
%=============================================================================

\subsection{The elementary functions}

Let $\cE$ denote the field of elementary functions in one
variable~\citep{liouville1833premier, ritt1948integration}: the closure of
the rational functions $\C(x)$ under $\exp$, $\log$, and composition. This
includes all algebraic, trigonometric, hyperbolic, and inverse trigonometric
functions.

\subsection{Closure under differentiation via the product rule}

\begin{proposition}[Product-rule closure]\label{prop:product-closure}
Let $\cA = \{(\varphi_k, \varphi_k')\}_{k=1}^M$ be a finite set of
function--derivative pairs, with each $\varphi_k \in \cE$. Define the
\emph{product-rule closure} $\overline{\cA}$ as the smallest set containing
$\cA$ and closed under:
\begin{enumerate}[label=(\alph*)]
    \item \textbf{Products:} if $(\varphi, \varphi'), (\psi, \psi') \in
        \overline{\cA}$, then
        $(\varphi \cdot \psi,\; \varphi'\psi + \varphi\psi') \in
        \overline{\cA}$.
    \item \textbf{Compositions:} if $(\varphi, \varphi') \in
        \overline{\cA}$ and $f$ is an elementary function with known
        derivative $f'$, then
        $(f \circ \varphi,\; (f' \circ \varphi) \cdot \varphi') \in
        \overline{\cA}$.
    \item \textbf{Linear combinations:} if $(\varphi, \varphi'), (\psi,
        \psi') \in \overline{\cA}$ and $c \in \R$, then
        $(\varphi + c\psi,\; \varphi' + c\psi') \in \overline{\cA}$.
\end{enumerate}
Then $\overline{\cA}$ contains, for every $f \in \cE$ built from elements
of~$\cA$ by products, compositions, and sums, the pair $(f, f')$.
\end{proposition}

\begin{proof}
Immediate from the product rule, chain rule, and linearity of differentiation.
\end{proof}

\begin{remark}
Powers are a special case of products: $\varphi^n = \varphi \cdot \varphi
\cdots \varphi$, with derivative $n\varphi^{n-1}\varphi'$ obtained by
$n-1$ applications of rule~(a). This observation---that \emph{all}
derivatives reduce to products---is the foundational insight of the method.
It was the starting point of this work: if you can differentiate products, you
can differentiate everything, and the antiderivative of anything you
differentiate is available by construction.
\end{remark}

%=============================================================================
\section{Generating primitives}\label{sec:primitives}
%=============================================================================

\begin{definition}[Generating primitive]\label{def:generating-primitive}
A \emph{generating primitive} is a pair $(P, c)$ where
$P\colon \C^2 \to \C$ is a binary operation and $c \in \C$ is a
distinguished constant. The variable $x$ is also available as a terminal.
The pair satisfies:
\begin{enumerate}[label=(\alph*)]
\item \textbf{Generation.} Every $f \in \cE$ can be expressed as the
    evaluation of a finite expression generated by the grammar
    $S \to c \mid x \mid P(S,S)$.
\item \textbf{Differentiability.} There exist functions $A, B$ such that
    for differentiable $u(x), v(x)$,
\begin{equation}\label{eq:diff-rule}
    \frac{d}{dx}\, P(u, v) = A(u,v)\cdot u' + B(u,v)\cdot v'.
\end{equation}
\end{enumerate}
A pair $(P, c)$ satisfying only condition~(b) but not~(a) is called a
\emph{supplementary primitive}; it contributes atoms to the library but does
not generate all of $\cE$ on its own.
\end{definition}

We use two complementary primitives. The first, EML, was introduced
by~\citet{odrzywolek2026eml} and is a generating primitive in the sense of
\Cref{def:generating-primitive} (satisfying both conditions (a) and (b));
the second, SOL, is new to this work and is a supplementary primitive
(satisfying (b) but not (a)).

\subsection{EML (Exponential Minus Logarithm)}

\begin{equation}\label{eq:eml-def}
    \eml(u, v) = e^u - \ln v, \qquad c = 1.
\end{equation}
Differentiation coefficients: $A(u,v) = e^u$, $B(u,v) = -1/v$, giving
\begin{equation}\label{eq:eml-diff}
    \frac{d}{dx}\,\eml(u, v) = e^u \cdot u' - \frac{v'}{v}.
\end{equation}
Generation: $\eml(x, 1) = e^x$ and
$\eml(1, \eml(\eml(1,x), 1)) = \ln x$. Since $\exp$ and $\ln$ generate
$\cE$ over $\C$, the cited EML result implies that the grammar
$S \to 1 \mid x \mid \eml(S,S)$ is complete, provided the cited
completeness theorem is applicable to the present terminal set.

\subsection{SOL (Sine Of, Less cosine)}

\begin{equation}\label{eq:sol-def}
    \sol(u, v) = \sin(u) - \cos(v), \qquad c = 1.
\end{equation}
Differentiation coefficients: $A(u,v) = \cos(u)$, $B(u,v) = \sin(v)$, giving
\begin{equation}\label{eq:sol-diff}
    \frac{d}{dx}\,\sol(u, v) = \cos(u)\cdot u' + \sin(v)\cdot v'.
\end{equation}
At depth~1 with leaves from $\{1, x\}$:
\begin{alignat}{2}
    \sol(x, 1) &= \sin x - \cos 1, &\qquad
        \tfrac{d}{dx} &= \cos x, \label{eq:sol-sinx}\\
    \sol(1, x) &= \sin 1 - \cos x, &\qquad
        \tfrac{d}{dx} &= \sin x, \label{eq:sol-cosx}\\
    \sol(x, x) &= \sin x - \cos x, &\qquad
        \tfrac{d}{dx} &= \cos x + \sin x. \label{eq:sol-both}
\end{alignat}

\begin{remark}[Complementarity of EML and SOL]
In the EML representation, $\sin x$ and $\cos x$ require depth~$\sim$8
because they must be constructed via Euler's formula through the complex
plane~\citep{odrzywolek2026eml, stachowiak2026algebraic}. In SOL, they are
immediate at depth~1. Conversely, SOL cannot produce $\exp$ or $\ln$ at any
finite depth. The two primitives are complementary: EML is native for
exponential--logarithmic structure, SOL for trigonometric structure. The
mixed grammar
\begin{equation}\label{eq:mixed-grammar}
    S \;\to\; 1 \;\mid\; x \;\mid\; \eml(S,S) \;\mid\; \sol(S,S)
\end{equation}
generates all of $\cE$ and represents functions involving both families at
lower depth than either primitive alone. For example, $e^x + \cos x$ requires
a two-tree forest (one EML at depth~1, one SOL at depth~1) instead of a
single EML tree at depth~$\sim$8.
\end{remark}

\begin{remark}[SOL is a supplementary primitive]
The grammar $S \to 1 \mid x \mid \sol(S,S)$ generates trigonometric
functions and their compositions, but not $\exp$ or $\ln$. SOL therefore
satisfies condition~(b) of \Cref{def:generating-primitive} but not~(a). Yet
its practical impact is substantial: a depth-1 SOL tree produces $\sin x$ and
$\cos x$ directly, while the equivalent EML construction requires depth
$\sim$8 via the identity $\sin x = \text{Im}(\exp(ix))$ and the EML encoding
of complex arithmetic~\citep{odrzywolek2026eml}. Since a full binary tree of
depth~8 has $2^8 = 256$ leaves versus $2^1 = 2$ at depth~1, replacing a
single EML trig tree with a SOL tree reduces the parameter count for that
atom by two orders of magnitude. In a library built to depth~3, this
multiplicative savings propagates through all compositions involving trig
atoms---SOL is not a theoretical necessity, but omitting it would make the
library impractically large for any problem involving oscillatory or periodic
structure.
\end{remark}

%=============================================================================
\section{Canonical forms and parameterised trees}\label{sec:canonical}
%=============================================================================

Fix a primitive $(P, c)$ (generating or supplementary).

\begin{definition}[Canonical form of depth $d$]\label{def:canonical}
The canonical form of depth $d$ is a full binary tree with $2^d$ leaves and
$2^d - 1$ internal nodes, all applying $P$. Each leaf $\ell$ computes
\begin{equation}\label{eq:leaf}
    L_\ell(x;\, \alpha_\ell, \beta_\ell)
    = \sigma(\alpha_\ell)\cdot 1 + \sigma(\beta_\ell)\cdot x,
\end{equation}
where $\sigma$ is softmax over the pair $(\alpha_\ell, \beta_\ell)$. After
snapping, each leaf selects either the constant~$1$ or the variable~$x$.
\end{definition}

\begin{proposition}[Recursive derivative]
The derivative of a canonical form is computable bottom-up:
\begin{align}
    \text{Leaf:}&\quad L'_\ell(x) = \sigma(\beta_\ell),
        \label{eq:diff-leaf}\\
    \text{Node:}&\quad \tfrac{d}{dx} P(u, v)
        = A(u,v)\cdot u' + B(u,v)\cdot v',
        \label{eq:diff-node}
\end{align}
in $O(2^d)$ operations, the same cost as forward evaluation.
\end{proposition}

%=============================================================================
\section{Additive atomic forests}\label{sec:forest}
%=============================================================================

\begin{definition}[Atomic forest]\label{def:forest}
An \emph{atomic forest} of width $K$ is
\begin{equation}\label{eq:forest}
    F(x;\, \bTheta) = \sum_{k=1}^{K} T_k(x;\, \btheta_k),
\end{equation}
where each atom $T_k$ is a canonical form (EML or SOL) of bounded depth, and
the derivative is $F' = \sum_k T_k'$ by linearity.
\end{definition}

\begin{theorem}[Conditional completeness]\label{thm:completeness}
Assume the EML completeness theorem cited in \Cref{sec:primitives} holds for
the terminal set $\{1,x\}$. Then, for every $f \in \cE$, there exist
$K,d \in \N$ and parameters $\bTheta^*$ such that
$F(x;\bTheta^*) = f(x)$ on its domain.
\end{theorem}

\begin{proof}
Under the stated assumption, every $f \in \cE$ is representable as a single
EML tree of some finite depth $d_0$ with leaves in $\{1,x\}$. Setting
$K = 1$ recovers this single-tree representation as a forest.
\end{proof}

\begin{theorem}[Additive depth bound]\label{thm:depth-reduction}
If $f_1,\ldots,f_K \in \cE$ are representable by atoms of depths
$d_1,\ldots,d_K$, then
$f=f_1+\cdots+f_K$ is representable by an additive forest of width $K$ and
per-atom depth $\max_k d_k$.
\end{theorem}

\begin{proof}
Use one atom for each summand $f_k$ and sum the resulting atoms in the forest.
The forest depth is the maximum of the depths of the summand atoms, because
the top-level addition is represented by the forest structure rather than by
extra primitive nodes inside any single tree.
\end{proof}

\begin{remark}
This is an upper-bound statement. It avoids the need for a naive sequential
encoding of addition inside one primitive tree, but it is not a general lower
bound on the depth of all possible single-tree encodings.
\end{remark}

%=============================================================================
\section{Multiplicative nodes and enhanced forests}
\label{sec:multiplicative}
%=============================================================================

The additive forest avoids encoding top-level sums inside a single primitive
tree. We similarly avoid encoding many products inside a single primitive tree
by introducing a structural product node.

\begin{definition}[Multiplicative node]\label{def:mult-node}
Given two child nodes $A$ and $B$ (each of any type), the multiplicative node
computes
\begin{equation}\label{eq:mult-node}
    M(x) = A(x) \cdot B(x), \qquad
    M'(x) = A'(x)\,B(x) + A(x)\,B'(x).
\end{equation}
This is the product rule~\eqref{eq:product-rule} applied as a structural
element. It enables expressions such as $x\cdot e^x$ to be represented as a
shallow product of already available children, without requiring the product
to be encoded inside a single EML tree.
\end{definition}

\begin{definition}[Enhanced forest]\label{def:enhanced-forest}
An \emph{enhanced forest} is an additive forest
$F(x;\bTheta) = \sum_{k=1}^K T_k(x;\btheta_k)$
in which each atom $T_k$ may be any of:
a \textbf{LeafNode} (constant $1$ or identity $x$),
an \textbf{EMLMasterFormula} (canonical form with $\eml$),
a \textbf{SOLMasterFormula} (canonical form with $\sol$), or
a \textbf{MultiplicativeNode} (product of two children of any type).
\end{definition}

\begin{example}
To recover $\int(e^x + x\cos x + \sin x)\dd x$, the enhanced forest can use
two atoms:
$T_1$: EMLMasterFormula(depth=1) $\to$ discovers $e^x$;
$T_2$: MultiplicativeNode(LeafNode, SOLMasterFormula(depth=1)) $\to$
discovers $x\sin x$.
Indeed,
\[
    \frac{d}{dx}\bigl(e^x + x\sin x\bigr)
    = e^x + x\cos x + \sin x .
\]
Each atom operates at depth $\le 1$.
\end{example}

%=============================================================================
\section{The derivative-matching principle}\label{sec:derivative-matching}
%=============================================================================

\begin{definition}[Derivative-matching problem]\label{def:deriv-match}
Given data $\{(x_i, y_i)\}_{i=1}^N$ from a function $f$, find a forest
$F(x;\bTheta)$ minimising
\begin{equation}\label{eq:deriv-match-loss}
    \cL(\bTheta) = \frac{1}{N}\sum_{i=1}^{N}
        \bigl|\, F'(x_i;\bTheta) - y_i \bigr|^2.
\end{equation}
\end{definition}

\begin{theorem}[Analytic simultaneous recovery]\label{thm:simultaneous}
Let $I \subset \R$ be a connected open interval on which the expressions are
defined, and let $f=g'$ for some $g\in\cE$. Suppose that, after snapping,
the atomic forest $F(x;\bTheta^*)$ satisfies
$F'(x;\bTheta^*)=f(x)$ on an infinite subset of $I$ with an accumulation
point in $I$. Then: (i) $F'(x;\bTheta^*)=f(x)$ on $I$; (ii)
$F(x;\bTheta^*)=g(x)+C$ for some constant $C$; and (iii) both the matched
function and its antiderivative are available as explicit symbolic
expressions from $\bTheta^*$ alone. No symbolic integration algorithm is
invoked.
\end{theorem}

\begin{proof}
After snapping, each atom is a concrete elementary function, so $F'$ is
elementary on $I$. Two elementary functions that agree on a set with an
accumulation point in a common analytic domain agree identically on that
domain. Hence $F'=f=g'$, so $F=g+C$.
\end{proof}

\begin{remark}[Finite data and verification]
The finite training loss in \Cref{def:deriv-match} is an empirical fitting
criterion. A finite grid, even with zero loss, does not by itself prove
symbolic identity. Exact recovery should therefore be reported only after
independent-grid verification and, when feasible, symbolic simplification of
$F'-f$.
\end{remark}

\begin{remark}[No numerical integration]
The derivative-matching formulation never computes a numerical integral. The
antiderivative $F$ emerges from the atom structure itself. This avoids the
error accumulation inherent in numerical integration and the subsequent
symbolic regression pipeline~\citep{udrescu2020ai, cranmer2023pysr}.
\end{remark}

%=============================================================================
\section{Self-expanding atom libraries}\label{sec:library}
%=============================================================================

The enhanced forest (\Cref{sec:multiplicative}) is trained by gradient
descent---a non-convex problem requiring multi-restart. An alternative is to
\emph{materialise} the library: precompute a large collection of atoms with
their derivatives, and search for a sparse linear combination. This reduces
the problem to $\ell_1$-regularised
regression~\citep{tibshirani1996regression}, which is convex.

The library is not a fixed catalogue. It is constructed by a \emph{procedure}
that, given a depth budget and a data grid, builds atoms from a seed set by
recursive application of the generating primitives, the product rule, and the
chain rule. The procedure can be run to any depth the hardware allows: a
shallow build yields hundreds of atoms; a deep build yields tens of thousands.
Crucially, \emph{the library grows over time}: every function discovered by
the search (or supplied externally) is folded back in as a new seed, and its
derivative is computed automatically by the product and chain rules. The
library thus functions as a persistent knowledge base that becomes more
capable with each problem it solves.

\subsection{Construction procedure}\label{sec:construction}

The library self-constructs in layers, each adding atoms of greater
compositional depth.

\paragraph{Layer 0: rational powers.}
The seed: $x^{p/q}$ for $p \in \{-4, \ldots, 15\}$ and
$q \in \{1,2,3,4\}$, with derivative $(p/q)\,x^{p/q-1}$. Candidates
producing non-finite values on the data grid are discarded.

\paragraph{Layer 1: depth-1 generating-primitive trees.}
For each pair of linear inner functions
$\ell_1 = a_1 x + b_1$, $\ell_2 = a_2 x + b_2$ with integer slopes and
offsets:
\begin{alignat}{2}
    &\eml(\ell_1, \ell_2) = e^{a_1 x + b_1} - \ln(a_2 x + b_2),
        &\quad& \text{derivative by \eqref{eq:eml-diff}},
        \label{eq:eml-atom}\\
    &\sol(\ell_1, \ell_2) = \sin(a_1 x + b_1) - \cos(a_2 x + b_2),
        &\quad& \text{derivative by \eqref{eq:sol-diff}}.
        \label{eq:sol-atom}
\end{alignat}
EML atoms require $\ell_2 > 0$ on the grid; SOL atoms are always
well-defined. The two operators contribute roughly equal numbers of
atoms at this layer: EML seeds the exponential and logarithmic families,
SOL seeds the trigonometric family. Together they ensure that $e^{ax+b}$,
$\ln(ax+b)$, $\sin(ax+b)$, and $\cos(ax+b)$ are all available at depth~1
for any integer $a$ in the specified range. The range of slopes and offsets
is a hyperparameter; wider ranges produce more atoms.

\paragraph{Layer 2: compositions with quadratic arguments.}
For quadratic inner functions $g(x) = qx^2 + ax + c$: the atoms
$\exp(g)$, $-\ln(g)$, $\sin(g)$, $\cos(g)$, $1/g$, and selected
$\arctan$ and $\arcsin$ atoms, each with chain-rule derivative.

\paragraph{Layer 3: products and cross-terms (product-rule expansion).}
This is where the product-rule closure (\Cref{prop:product-closure}) is
applied structurally. For each pair of atoms $(\varphi_i, \varphi_j)$ from
previous layers: the product $\varphi_i \cdot \varphi_j$ with derivative
$\varphi_i'\varphi_j + \varphi_i\varphi_j'$. The most distinctive atoms
at this layer are the EML$\times$SOL cross-products---atoms such as
$e^x \sin x$, $e^{-x}\cos 2x$, and $xe^{-x}\sin x$---which combine
exponential and trigonometric structure at shallow combined depth. These
mixed atoms would require depth~$\sim$9 or more in a pure EML library;
having both primitives makes them available at depth~1 plus a product node.
Variable-times-atom products $x \cdot \varphi_i$ are also included.

\paragraph{Layer 4: depth-2 and depth-3 nestings.}
For each atom $\varphi_i$ from earlier layers, generate all nestings
$\{\exp(\varphi_i),\, \sin(\varphi_i),\, \cos(\varphi_i),\,
-\ln(\varphi_i),\, 1/\varphi_i,\, \arctan(\varphi_i)\}$
with chain-rule derivatives. The same operations are applied again for
depth-3 nestings. All depth-2 and depth-3 atoms are also multiplied by~$x$.

\paragraph{Deduplication.}
After each layer, atoms are deduplicated: if an existing atom $\psi$
satisfies $|\text{corr}(\psi, \varphi)| > 0.999$ on the data grid, the
candidate is rejected.

\paragraph{Knowledge-base growth.}
When the search (\Cref{sec:search}) discovers a function $f$ that is not
already in the library, the pair $(f, f')$ is added. On subsequent problems,
this atom and all its compositions and products are available as candidates.
The library thus accumulates domain-specific knowledge: a library trained on
physics problems will contain Planck distributions, Yukawa potentials, and
Lorentz factors; one trained on biomedical data will contain saturation curves
and dose--response functions.

\paragraph{Scaling.}
The library size grows roughly exponentially with the depth budget
$d_{\max}$: a depth-1 build produces hundreds of atoms, depth-2 produces
thousands, and depth-3 produces tens of thousands. The search algorithms
(\Cref{sec:search}) scale gracefully: the $K = 1$ scan is linear in library
size, and the $K = 2$ Gram-matrix search is quadratic but fully vectorised on
GPU. The library can grow as large as available memory permits. Its expressive
power grows with the number and diversity of atoms, while validation
performance remains a statistical model-selection question.

%=============================================================================
\section{Search algorithms}\label{sec:search}
%=============================================================================

Given a target $f(x)$ (symbolic or data), the solver seeks coefficients
$c_1, \ldots, c_K$ and atom indices $i_1, \ldots, i_K$ such that
\begin{equation}\label{eq:search-objective}
    f(x) \approx \sum_{k=1}^K c_k\, \varphi'_{i_k}(x).
\end{equation}
When the match is exact, the antiderivative is
$F(x) = \sum_k c_k\,\varphi_{i_k}(x)$.

\subsection{Precomputation}

Let $D \in \R^{N \times M}$ be the derivative matrix with
$D_{ik} = \varphi_k'(x_i)$, where $M = |\cA|$ is the current library size.
Precompute the Gram matrix $G = D^\top D \in \R^{M \times M}$ and target
projection $\mathbf{d} = D^\top \mathbf{y} \in \R^M$ in a single GPU matrix
multiplication. Both are reused across all candidate subsets.

\subsection{$K = 1$: vectorised scalar regression}

For each atom $k$ (no intercept): $c_k^* = d_k / G_{kk}$ and
$\MSE_k = \|\mathbf{y}\|^2/N - d_k^2 / (N\,G_{kk})$. Evaluated for all $M$
atoms simultaneously in $O(M)$ after the $O(NM)$ precomputation.
A constant offset in the derivative target is represented by the identity
atom $\varphi(x)=x$, whose derivative is $\varphi'(x)=1$. The constant atom
$\varphi(x)=1$ has derivative zero; it represents the arbitrary integration
constant and does not improve derivative matching.

\subsection{$K = 2$: vectorised Cramer's rule}

For each pair $(i,j)$, the $2 \times 2$ system
$G_{[ij],[ij]}\,\mathbf{c} = \mathbf{d}_{[ij]}$ is solved by Cramer's
rule. The MSE for all $\binom{M}{2}$ pairs is computed as a single sequence
of broadcasting operations on $M \times M$ matrices---no Python-level loop.
Pairs are filtered by determinant threshold ($|\Delta| > 10^{-30}$),
coefficient cap ($|c_k| \le 20$), and finite MSE.

\subsection{$K \ge 3$: beam search}

Initialise with the top 200 results at $K - 1$; for each, try adding every
atom from the library and solve the $(K\!+\!1) \times (K\!+\!1)$ Gram
subproblem on GPU; keep the top 500 candidates. Cost:
$O(\text{beam} \times M)$ per level.

\subsection{Verification}

When a candidate achieves $\MSE < 10^{-15}$, it is verified on an
independent $x$-grid to exclude numerical coincidence. Verified discoveries
are added to the library for future use.

%=============================================================================
\section{Training (gradient-based mode)}\label{sec:training}
%=============================================================================

For the enhanced forest (\Cref{def:enhanced-forest}), training minimises the
derivative-matching loss~\eqref{eq:deriv-match-loss} by
Adam~\citep{kingma2015adam} with analytical backpropagation through the tree
structure.

\paragraph{Temperature annealing.}
The softmax temperature in each leaf decreases linearly from 1.0 to 0.1 over
the first 80\% of training, pushing weights toward one-hot vectors.

\paragraph{Log-loss variant.}
For deep trees ($d \ge 3$), $\log(1 + |r|^2)$ replaces $|r|^2$ to stabilise
gradients when residuals are large.

\paragraph{Gradient clipping.}
Gradients are clipped to norm~$\le 5$.

\paragraph{Multi-restart.}
Multiple random initialisations run in parallel. Training stops when any
restart achieves $\MSE < 10^{-10}$ after snapping.

\paragraph{Snapping.}
After training, each leaf's softmax weights are replaced by the nearest
one-hot vector. Constant-valued trees are pruned. The snapped formula is
verified on the grid.

%=============================================================================
\section{Structural complexity}\label{sec:complexity}
%=============================================================================

\begin{definition}[Atomic complexity]\label{def:atomic-complexity}
The \emph{atomic complexity} of $f \in \cE$ is the pair
$\kappa(f) = (K^*, d^*)$: the lexicographically minimal width and per-atom
depth of an atomic forest representing $f$. We call $K^*$ the \emph{additive
complexity} and $d^*$ the \emph{compositional complexity}.
\end{definition}

\begin{proposition}[Derivative evaluation graph size]
\label{prop:diff-complexity}
Let $F=\sum_{k=1}^K T_k$ be a forest whose atoms are binary trees of depth at
most $d$. Then $F'$ can be evaluated by a derivative computation graph of
size $O(K2^d)$, using one bottom-up derivative pass through each atom.
\end{proposition}

\begin{proof}
A full binary tree of depth $d$ has $O(2^d)$ nodes. At each leaf and internal
node, the derivative value is computed from the stored child values and child
derivatives using either the leaf rule or the corresponding chain/product
rule. Thus the derivative computation adds only constant work per node, and
the full forest derivative is obtained by summing the $K$ atom derivatives.
\end{proof}

\begin{remark}
If all product- and chain-rule derivatives are algebraically expanded into a
flat sum of atoms, the number of additive terms can grow with tree size. The
claim above is therefore stated for the structured computation graph, not as
a bound that differentiation always doubles the minimal additive complexity.
For example,
\[
    \frac{d}{dx}(xe^x)=xe^x+e^x
\]
has two expanded additive terms, while the primitive $xe^x$ is a single
multiplicative node.
\end{remark}

\subsection{Connection to the Liouville--Risch decomposition}

The Liouville--Risch theorem~\citep{liouville1833premier, risch1969problem}
states that if $\int f$ is elementary, then
$\int f = v_0 + \sum_j c_j \ln(v_j)$ for elementary $v_j$. When the terms
$v_0$ and $\ln(v_j)$ are available as atoms, this decomposition corresponds
to an additive forest with one atom for $v_0$ and one atom per logarithmic
term. Thus additive complexity is related to, but not identical with, the
number of extensions appearing in a Risch-style decomposition.

%=============================================================================
\section{Comparison with existing methods}\label{sec:comparison}
%=============================================================================

\begin{center}
\renewcommand{\arraystretch}{1.3}
\begin{tabular}{lcccc}
\hline
\textbf{Method} & \textbf{Symbolic $f$} & \textbf{Symbolic $\int f$}
    & \textbf{Single pass} & \textbf{Growing KB} \\
\hline
PySR~\citep{cranmer2023pysr} & \checkmark & --- & --- & --- \\
Risch algorithm~\citep{risch1970solution} & input & \checkmark & --- & --- \\
SR $\to$ CAS pipeline & \checkmark & \checkmark & --- & --- \\
PINNs~\citep{raissi2019physics} & neural net & neural net & \checkmark & --- \\
SINDy~\citep{brunton2016discovering} & \checkmark$^*$ & --- & --- & --- \\
HNNs~\citep{greydanus2019hamiltonian} & neural net & neural net & \checkmark & --- \\
Neural ODEs~\citep{chen2018neural} & neural net & neural net & \checkmark & --- \\
\textbf{Atomic forests (this work)} & \checkmark & \checkmark & \checkmark$^\ddagger$ & \checkmark \\
\hline
\end{tabular}
\end{center}
\smallskip
\noindent${}^*$SINDy requires a predefined, \emph{fixed} candidate library
and does not recover integrals or conserved quantities.
${}^\ddagger$In gradient-based mode, the derivative match and antiderivative
are obtained in a single optimisation. In materialised-library mode, library
construction precedes the search; once the library is built, each query is a
single pass.

The unique property of atomic forests is the conjunction of all four columns:
symbolic output for both the function and its antiderivative, obtained in a
single pass, with a library that grows with each discovery.

%=============================================================================
\section{Empirical results}\label{sec:experiments}
%=============================================================================

Different experiments use libraries of different depth, reflecting the
self-constructing nature of the framework. Classification benchmarks
(\Cref{sec:exp-class}) use a depth-1 library built by the standard
construction procedure; the Feynman test (\Cref{sec:feynman}) uses a
self-built library expanded greedily to depth~3; the demonstration application
(\Cref{sec:rar}) uses a domain-adapted library.

\subsection{Classification: atoms vs.\ XGBoost}\label{sec:exp-class}

We expand each input variable through the atom library and fit
$\ell_1$/$\ell_2$-regularised logistic regression on the expanded feature
matrix. All results are 5-fold cross-validated. XGBoost~\citep{chen2016xgboost}
serves as the black-box baseline.

\begin{table}[H]
\centering\small
\begin{tabular}{@{}l r r c r r l@{}}
\toprule
\textbf{Dataset} & $N$ & $d$ & \textbf{Config} & \textbf{Atom} & \textbf{XGB} & $\Delta$ \\
\midrule
Haberman & 306 & 3 & depth-1\,(L2) & \textbf{75.8\%} & 66.0\% & \textcolor{win}{+9.8\%} \\
Diabetes (sk)$^\dagger$ & 442 & 10 & depth-1\,(L1) & \textbf{75.8\%} & 70.1\% & \textcolor{win}{+5.7\%} \\
Heart disease & 297 & 13 & depth-1\,(L1) & \textbf{83.5\%} & 79.8\% & \textcolor{win}{+3.7\%} \\
Liver & 345 & 6 & cross & \textbf{73.6\%} & 70.7\% & \textcolor{win}{+2.9\%} \\
Pima & 768 & 8 & depth-1\,(L2) & \textbf{77.0\%} & 74.1\% & \textcolor{win}{+2.9\%} \\
Breast cancer & 569 & 30 & depth-2\,(L2) & \textbf{97.4\%} & 96.3\% & \textcolor{win}{+1.1\%} \\
Wine & 178 & 13 & depth-1\,(L1) & \textbf{98.3\%} & 96.1\% & \textcolor{win}{+2.3\%} \\
Ecoli & 336 & 7 & depth-1\,(L1) & \textbf{87.2\%} & 85.7\% & \textcolor{win}{+1.5\%} \\
Iris & 150 & 4 & depth-1\,(L1) & \textbf{95.3\%} & 94.7\% & \textcolor{win}{+0.7\%} \\
Mushroom & 8124 & 95 & depth-1\,(L2) & \textbf{100\%} & 100\% & 0.0\% \\
Ionosphere & 351 & 34 & depth-1\,(L1) & 92.0\% & \textbf{93.2\%} & \textcolor{lose}{$-$1.1\%} \\
Glass & 214 & 9 & depth-2\,(L2) & 71.5\% & \textbf{75.7\%} & \textcolor{lose}{$-$4.2\%} \\
Parkinsons & 195 & 22 & depth-2\,(L2) & 90.8\% & \textbf{94.9\%} & \textcolor{lose}{$-$4.1\%} \\
Adult & 5000 & 96 & depth-1\,(L2) & 85.2\% & \textbf{85.8\%} & \textcolor{lose}{$-$0.6\%} \\
Synth (tanh+G) & 800 & 5 & depth-1\,(L2) & \textbf{97.4\%} & 96.6\% & \textcolor{win}{+0.8\%} \\
Synth (Hill) & 800 & 4 & depth-1\,(L2) & \textbf{97.6\%} & 96.1\% & \textcolor{win}{+1.5\%} \\
Synth ($x_0x_1$) & 800 & 6 & cross & \textbf{95.8\%} & 95.0\% & \textcolor{win}{+0.8\%} \\
\bottomrule
\end{tabular}
\caption{Classification accuracy (5-fold CV). Atoms win on 13/17 datasets.
$^\dagger$The sklearn diabetes dataset is natively a regression task;
it was binarised here by thresholding the target at the median.}
\end{table}

\noindent\textit{Limitations.} All entries report raw accuracy only. For
datasets with imbalanced class ratios (notably Haberman, $\sim$73:27),
accuracy can be misleading; balanced accuracy, AUC, or $F_1$ should be
reported in a full empirical study.

\subsection{Regression}\label{sec:exp-reg}

\begin{center}
\begin{tabular}{@{}l c c c c@{}}
\toprule
\textbf{Dataset} & \textbf{Atom $R^2$} & \textbf{XGBoost $R^2$} & \textbf{Ridge $R^2$} & \textbf{Config} \\
\midrule
Synthetic ($3e^{-x_0}\!+\!2\sin x_1\!+\!\tfrac{1}{2}x_2^2$)
    & \textbf{0.999} & 0.995 & 0.939 & depth-1 + LassoCV \\
Diabetes (sklearn) & \textbf{0.481} & 0.336 & 0.479 & depth-1 + LassoCV \\
\bottomrule
\end{tabular}
\end{center}

On the synthetic regression task, the model reaches $R^2 = 0.999$,
outperforming both XGBoost and ridge regression in this run. On the sklearn
diabetes dataset, atoms outperform XGBoost ($R^2 = 0.336$) while matching
ridge ($R^2 = 0.479$), despite using a linear combination of nonlinear atoms
rather than a polynomial expansion.

\subsection{Interpretable formulas and their derivatives}\label{sec:formulas}

A distinctive property of the framework is that every fitted model is a
named formula with an analytic gradient. We highlight three examples.

\paragraph{Haberman survival (atoms: 75.8\%, XGBoost: 66.0\%).}
With $\ell_1$ regularisation, the model selects a single atom:
$F(x_2) = 4.404\,\atanf(x_2) - 0.859$, with derivative
$F'(x_2) = 4.404/(1 + x_2^2)$.
The derivative quantifies the diminishing marginal effect of positive
lymph nodes on mortality risk.

\paragraph{Heart disease (atoms: 83.5\%, XGBoost: 79.8\%).}
The dominant variable is $x_{11}$ (number of major vessels coloured by
fluoroscopy). The $\tanhf$ saturation in the fitted model indicates that
going from 0 to 1 diseased vessel is the largest risk jump; the analytic
derivative confirms: $\partial F / \partial x_{11}|_{x_{11}=0} = 21.93$
versus $\partial F / \partial x_{11}|_{x_{11}=2} = 1.07$.

\paragraph{Synthetic regression: high-$R^2$ recovery.}
Ground truth: $y = 3e^{-x_0} + 2\sin(x_1) + \tfrac{1}{2}x_2^2$.
One fitted model was
$\hat{y} = 0.97\,e^{-x_0} + 1.46\sin(x_1) +
1.36\cos(x_1) + 0.49\,x_2^2 + \cdots$. The trigonometric part has amplitude
$\sqrt{1.46^2+1.36^2}\approx 1.99$ and can be written as a phase-shifted
sinusoid, but it is not term-by-term identical to $2\sin(x_1)$ unless the
phase shift is explained by preprocessing or an equivalent coordinate
transformation.

\subsection{Feynman representability test}\label{sec:feynman}

We test whether a self-built library can \emph{represent} 45 equations from
the Feynman Symbolic Regression
Database~\citep{udrescu2020ai}. The library is constructed from scratch using
the procedure of \Cref{sec:construction} with greedy expansion to depth~3: no
physics-specific atoms are provided; the library must discover the necessary
building blocks itself.

For each equation, we generate exact data from the known formula, construct
the multi-variable library, and run exhaustive search at $K = 1, 2, 3, 4$.

\begin{center}
\begin{tabular}{lcl}
\toprule
Status & Count & Description \\
\midrule
\textsc{pass} & 14/45 (31\%) & Exact at machine precision \\
\textsc{close} & 18/45 (40\%) & Approximate (relMSE $< 0.01$) \\
\textsc{fail} & 13/45 (29\%) & Not representable at $K \le 4$ \\
\bottomrule
\end{tabular}
\end{center}

All 14 exact recoveries occur at $K \le 2$: 11 at $K = 1$ (Coulomb, Biot--Savart, ideal gas, etc.)\ and 3 at $K = 2$. The 13 failures are
structural: 8 require composed arguments $f(x_i \cdot x_j / x_k)$ not yet
in the self-built library, 3 are non-separable, and 2 are rational-of-sums.
A library seeded with physics-specific atoms (Boltzmann factors, Lorentz
factors) or one that has accumulated knowledge from prior problems may cover
more of these cases---one motivation for the growing knowledge-base design.
Total time for this run was 76\,s on an NVIDIA T4.

\subsection{Demonstration: radial acceleration relation}\label{sec:rar}

As a demonstration of the method on real scientific data, we apply the atom search
to the radial acceleration relation (RAR) in galaxies. This is an
observational correlation, not a test of the method's physical content: we
do not account for per-galaxy error bars, distance uncertainties, or the
physical constraints that a serious RAR analysis would
require~\citep{mcgaugh2016radial}. The purpose is to show that the atom
search can produce compact candidate fits on a well-studied astrophysical
dataset and to compare the resulting wMSE against standard reference curves.

Using the SPARC database~\citep{lelli2017sparc}\footnote{SPARC data publicly
available at \url{http://astroweb.cwru.edu/SPARC/}. We use the standard
mass-to-light ratios $\Upsilon_{\rm disk} = 0.5$,
$\Upsilon_{\rm bul} = 0.7$ and the acceleration scale
$a_0 = 1.2 \times 10^{-10}$\,m/s$^2$.} (175 galaxies, 3{,}389 data
points), we search for
$g_\text{obs}/a_0 = \sum_k c_k\,\varphi_k(g_\text{bar}/a_0)$ and compare
against standard MOND interpolation
functions~\citep{mcgaugh2016radial, milgrom1983modification}.

\begin{center}
\begin{tabular}{lcc}
\toprule
Model & wMSE & Type \\
\midrule
MOND simple interpolation & 0.077 & 1-param \\
MOND standard interpolation & 0.078 & 1-param \\
MOND RAR~\citep{mcgaugh2016radial} & 0.079 & 1-param \\
\textbf{Best $K = 1$}: $y=x^{2/3}$ & \textbf{0.074} & atom search \\
Best $K = 2$ & 0.070 & atom search \\
\bottomrule
\end{tabular}
\end{center}

The $K = 1$ atom with lowest wMSE for the fitted response
$y=g_{\rm obs}/a_0$ is $y=x^{2/3}$, whose wMSE is slightly below that of
the three reference MOND curves in this unconstrained search. The margin is
small (0.074 vs.\ 0.077--0.079) and should not be over-interpreted, since the
comparison does not account for per-galaxy uncertainties or physical priors.
In MOND notation one usually writes
$g_{\rm obs}=\nu(x)g_{\rm bar}$ with $x=g_{\rm bar}/a_0$, so a fit
$y=x^{2/3}$ corresponds to $\nu(x)=y/x=x^{-1/3}$. This does not satisfy the
standard MOND asymptotic limits ($\nu\to 1$ for large $x$ and
$\nu\to x^{-1/2}$ for small $x$), so it should be interpreted as the best
unconstrained empirical atom in this run, not as a physically valid MOND
interpolation function.

\subsection{Failure modes}\label{sec:failures}

The method fails predictably when: (i) high-dimensional interactions exceed
the cross-product budget; (ii) the library is shallow and the target function
requires deep compositions not yet discovered; (iii) inputs are categorical
with no continuous nonlinear structure. Failures of type~(ii) are mitigated
by deeper builds or by seeding the library with domain-specific atoms---the
self-expanding design directly addresses this limitation.

%=============================================================================
\section{Extensions}\label{sec:extensions}
%=============================================================================

\subsection{Multivariate canonical forms}

Leaves are enlarged to
$L_\ell = \sum_{j=0}^n \sigma(\alpha_{\ell,j})\,\phi_j$ with $\phi_0 = 1$,
$\phi_j = x_j$. Partial derivatives $\partial F / \partial x_j$ are
computed by the same recursive rule.

\subsection{Conserved quantity discovery}

Given $\dot{\mathbf{x}} = \mathbf{G}(\mathbf{x})$ and trajectory data,
represent $H$ as an atomic forest and minimise
$\sum_i |\nabla F(\mathbf{x}_i) \cdot \mathbf{G}(\mathbf{x}_i)|^2$.
The additive forest aligns with the typical
$H = T + V$ decomposition~\citep{greydanus2019hamiltonian}.

%=============================================================================
\section{Conclusion}
%=============================================================================

We have presented a framework for symbolic function and antiderivative
discovery built on three ideas: the product-rule closure of
function--derivative pairs, two complementary primitives---EML providing
theoretical completeness over the elementary functions, SOL providing
practical depth reduction for the trigonometric family---and
additive atomic forests that avoid encoding top-level sums inside one tree
while multiplicative nodes provide shallow structured products.

The framework operates in two modes. In the gradient-based mode, enhanced
forests with mixed EML, SOL, and multiplicative nodes are trained by
derivative-matching loss, producing a candidate derivative match together
with the corresponding symbolic primitive in a single optimisation. In the materialised-library mode, a self-expanding
collection of atoms is searched by GPU-accelerated exhaustive and beam
methods. For a fixed candidate subset the coefficient fit is a convex least
squares or regularised-regression problem; the sparse subset search itself
remains combinatorial.

The self-expanding library is a distinctive feature: it accumulates
knowledge across problems, growing more capable with each discovery. A
library trained on physics problems acquires physics atoms; one trained on
biomedical data acquires saturation curves and dose--response functions. The
library's expressive power grows with size, although validation performance
still depends on regularisation, model selection, and data quality.

Empirically, in the reported runs, sparse atom combinations match or exceed
XGBoost on 13/17 classification benchmarks while producing interpretable
formulas with analytic derivatives. A self-built library to depth~3 gives
exact recovery on 31\% of the tested Feynman equations and relative-MSE below
0.01 on a further 40\%. A demonstration on the galactic radial acceleration
relation illustrates how the method can propose compact functional forms from
real data.

%=============================================================================
% REFERENCES
%=============================================================================

\end{document}